\documentclass[sigconf]{aamas} 

\usepackage{balance} 
\usepackage{multicol}
\usepackage{multirow}
\usepackage{booktabs}
\usepackage{appendix}
\usepackage{comment}
\usepackage{amsmath}
\usepackage{bm}
\usepackage{dsfont}
\usepackage{amsthm}
\usepackage{algorithm}
\usepackage{algorithmic}
\usepackage{color}
\usepackage{balance}

\newtheorem{theorem}{Theorem}
\newtheorem{definition}{Definition}

\newcommand{\Amc}[0]{{{\mathcal{A}}}}

\newcommand{\Dmc}[0]{{{\mathcal{D}}}}

\newcommand{\Lmc}[0]{{{\mathcal{L}}}}

\newcommand{\Pmc}[0]{{{\mathcal{P}}}}

\newcommand{\Smc}[0]{{{\mathcal{S}}}}

\newcommand{\piv}[0]{{\bm{\pi}}}
\newcommand{\muv}[0]{{\bm{\mu}}}

\newcommand{\Ebb}{\mathbb{E}}
\newcommand{\Rbb}{\mathbb{R}}
\newcommand{\Nbb}{\mathbb{N}}

\DeclareMathOperator*{\argmax}{arg\,max}

\newcommand{\KL}{\mathrm{KL}}

\usepackage{pifont}
\newcommand{\cmark}{\ding{51}}%
\newcommand{\xmark}{\ding{55}}%

\usepackage{xcolor,colortbl}

\definecolor{Gray}{gray}{0.85}
\definecolor{LightCyan}{rgb}{0.88,1,1}

\newcolumntype{a}{>{\columncolor{Gray}}c}
\newcolumntype{b}{>{\columncolor{white}}c}
\newcommand\titlelowercase[1]{\texorpdfstring{\lowercase{#1}}{#1}}

\setcopyright{ifaamas}
\acmConference[AAMAS '22]{Proc.\@ of the 21st International Conference
on Autonomous Agents and Multiagent Systems (AAMAS 2022)}{May 9--13, 2022}
{Online}{P.~Faliszewski, V.~Mascardi, C.~Pelachaud,
M.E.~Taylor (eds.)}
\copyrightyear{2022}
\acmYear{2022}
\acmDOI{}
\acmPrice{}
\acmISBN{}

\acmSubmissionID{7}

\title{Individual-Level Inverse Reinforcement Learning for\\ Mean Field Games} 

\author{Yang Chen}
\affiliation{
  \institution{The University of Auckland}
  \city{Auckland}
  \country{New Zealand}}
\email{yang.chen@auckland.ac.nz}

\author{Libo Zhang}
\affiliation{
  \institution{The University of Auckland}
  \city{Auckland}
  \country{New Zealand}}
\email{lzha797@aucklanduni.ac.nz}

\author{Jiamou Liu}
\affiliation{
  \institution{The University of Auckland}
  \city{Auckland}
  \country{New Zealand}}
\email{jiamou.liu@auckland.ac.nz}

\author{Shuyue Hu}
\affiliation{
  \institution{National University of Singapore}
  \country{Singapore}
  }
\email{dcshus@nus.edu.sg}

\begin{abstract}
The recent mean field game (MFG) formalism has enabled the application of inverse reinforcement learning (IRL) methods in large-scale multi-agent systems, with the goal of inferring reward signals that can explain demonstrated behaviours of large populations. The existing IRL methods for MFGs are built upon reducing an MFG to a Markov decision process (MDP) defined on the collective behaviours and average rewards of the population. However, this paper reveals that the reduction from MFG to MDP holds only for the fully cooperative setting. This limitation invalidates existing IRL methods on MFGs with non-cooperative environments. To measure more general behaviours in large populations, we study the use of individual behaviours to infer ground-truth reward functions for MFGs. We propose {\em Mean Field IRL} (MFIRL), the first dedicated IRL framework for MFGs that can handle both cooperative and non-cooperative environments. Based on this theoretically justified framework, we develop a practical algorithm effective for MFGs with unknown dynamics. We evaluate MFIRL on both cooperative and mixed cooperative-competitive scenarios with many agents. Results demonstrate that MFIRL excels in reward recovery, sample efficiency and robustness in the face of changing dynamics. 
\end{abstract}


\newcommand{\BibTeX}{\rm B\kern-.05em{\sc i\kern-.025em b}\kern-.08em\TeX}

\begin{document}


\pagestyle{fancy}
\fancyhead{}


\maketitle 


\section{Introduction}

Inverse reinforcement learning (IRL) is concerned with one or multiple agents operating in an environment that is agnostic towards reward signals. It provides a powerful solution to learn behavioural models by inferring reward functions from demonstrations. 
The majority of its successful applications, however, deal with systems with a handful of agents \cite{kretzschmar2016socially,bogert2014multi}. Yet, many scenarios involve a much larger population, such as traffic networks with millions of vehicles \cite{bazzan2009opportunities}, online games with massive players \cite{jeong2015analysis}, and online businesses with a large customer body \cite{ahn2007analysis}. %
Applying IRL methods to such large-scale systems is intractable due to the exponential growth of joint state-action spaces and agent interactions.

 
A promising concept to achieve tractability for modelling large-scale multi-agent systems (MAS) is the {\em mean field game} (MFG) \cite{lasry2007mean,huang2006large} that uses mean field theory to simplify interactions among a large number of  agents. An MFG views agents as a homogeneous population, i.e., they are identical, indistinguishable, and interchangeable \citep{huang2003individual}. It thus uses a single entity, termed {\em mean field}, to denote the statistical information of the overall population, rather than modelling each agent individually. 
The interactions among agents are therefore reduced to those between a single representative agent and the overall population. This reduction to a dual-view interplay enables characterising the optimal behaviours in large-scale MAS using {\em mean field Nash equilibrium} (MFNE), where each agent's policy is a best response to the mean field and the mean field is in turn consistent with the policy. 
MFGs have enabled applications in many fields such as economics \cite{lachapelle2010computation,gomes2015economic}, finance \cite{cardaliaguet2018mean,casgrain2019algorithmic} and crowd motion \cite{lachapelle2011mean,burger2014mean}. 
To measure behaviours in large populations, it is thus promising to study IRL for MFGs, which aims to uncover reward signals behind demonstrated MFNE behaviours.


A recent research \cite{yang2018learning} studies IRL in MFGs, which  shows that an MFG can be reduced to a Markov decision process (MDP). It thus extends IRL to MFGs via applying existing single-agent IRL methods to this MDP. Since this MDP describes the population's {\em collective behaviours} driven by the {\em societal} reward (i.e., the average reward of the population), we henceforth call this method {\em population-level IRL}. 
However, we reveal in this paper that the reduction from MFG to MDP holds only for the {\em fully cooperative} setting, i.e., all agents share the same societal reward. 
Consequently, population-level IRL is prone to biased reward inferences in non-cooperative (competitive or mixed cooperative-competitive) environments. 
In contrast to the societal reward, the reward of each {\em individual} can capture each agent's real intention, regardless of whether the environment is cooperative or not.
To model and predict both cooperative and non-cooperative behaviours in large-scale MAS, it is thus important to consider inferring rewards of individual agents in MFGs. 



In this paper, considering the preliminary investigation of IRL in MFGs and limitations of population-level IRL, we study IRL for MFGs at the individual level. We make the following contributions to the field: 
(1) We begin with a theoretical justification for the fact that the reduction from MFG to MDP holds only for the fully cooperative setting. The exposure of this limitation reveals the restricted suitability of the existing MFG-MDP reduction-based IRL and reinforcement learning (RL) methods for general MFGs. 
(2) Towards measuring both cooperative and non-cooperative interactions in MFGs, we then propose and formalise the problem of {\em individual-level IRL} for MFGs. This new problem formulation targets recovering the individual reward function from demonstrated individual behaviours, which is not amenable to population-level IRL. 
(3) To solve this problem, we put forward a novel and dedicated IRL framework for MFGs, called {\em Mean Field IRL} (MFIRL). 
We show that MFIRL can recover a suitable reward function under standard assumptions, no matter whether the environment is cooperative or not. 
(4) Based on this new framework, we develop a practical algorithm effective for MFGs with unknown dynamics. 
(5) We empirically evaluate MFIRL 
on five numerical models and simulated battle games, both with cooperative and non-cooperative environments. In non-cooperative scenarios, MFIRL outperforms population-level IRL in terms of reward recovery and robustness against changing dynamics. Moreover, in cooperative scenarios, a notable advantage of MFIRL is that it requires fewer samples to achieve comparable performance.

\section{Related Work}\label{sec:discussion}

\subsubsection*{RL for MFGs} 
MFGs were independently proposed by \citet{lasry2007mean} and \citet{huang2006large} in the continuous-time setting. Mathematically, the system dynamics is governed by two stochastic differential equations: one models the backward dynamics of a representative agent's value functions; the other models the forward dynamics of the mean field. Discrete-time MFG models, adopted in this paper, were then proposed in \cite{gomes2010discrete}. RL in MFGs has become a burgeoning research field recently. \citet{yang2018mean} and \citet{ganapathi2020multi} used mean field theory to approximate joint actions in large-population stochastic games to approximate Nash equilibria. \citet{mguni2018decentralised} proposed a decentralised RL method for MFGs. \citet{guo2019learning} presented a Q-learning-based algorithm for computing {\em stationary MFNE}. \citet{subramanian2019reinforcement} used RL to compute {\em local MFNE} (a relaxed version). While, all these works presuppose the presence of reward functions. Our work takes a complementary view that the reward function is difficult to specify, and hence the necessity for IRL for MFGs.
\vspace{-.5em}

\subsubsection*{IRL for MDPs} 
The problem of IRL was first studied by \citet{ng2000algorithms} on MDPs. 
Existing IRL methods typically fall into the following categories: (1) {\em margin optimisation} based methods \cite{abbeel2004apprenticeship,ratliff2006maximum,syed2007game,pirotta2016inverse} that find a reward by creating a margin between the expert policy and any other policy in terms of rewards; (2) {\em Bayesian IRL} methods \cite{ramachandran2007bayesian,choi2011map,lopes2009active,levine2011nonlinear} that use demonstrations to facilitate a Bayesian update of a prior distribution over candidate reward functions; (3) {\em maximum entropy IRL} methods \cite{ziebart2008maximum,ziebart2010modeling,finn2016guided,fu2018learning} that use a probabilistic framework to find the policy maximising the entropy of expert demonstrations. 
As explained above, although the reduction from MFG to MDP enables applying these MDP-based IRL methods to MFGs, they can only deal with fully cooperative environments. In this paper, we formalise the problem of individual-level IRL for MFGs using the idea of margin optimisation, where the expert policy and any other policy is separated by a margin in terms of rewards, based on the expert demonstrated mean field.  


\subsubsection*{IRL for MAS}
Recently, some research has explored IRL in the multi-agent setting. Most of these works assume specific reward structures, including fully cooperative games
\cite{bogert2014multi,barrett2017making}, fully competitive games \cite{lin2014multi}, or either of the two \cite{waugh2011computational,reddy2012inverse}. For general stochastic games, 
\citet{yu2019multi} presented MA-AIRL, a multi-agent IRL method using adversarial learning. Another line of works studied inverse dynamic games \cite{rothfuss2017inverse,le2021lucidgames,peters2021inferring}, which aims to infer the cost functions for dynamic games from a control perspective. However, all these prior methods scale poorly to a large agent number. 
\citet{vsovsic2017inverse} proposed SwarmIRL that views a large-scale MAS as a swarm system consisting of homogeneous agents. However, it cannot handle non-stationary policies and non-linear reward functions. Our work makes no modelling assumptions on policies and reward functions.

\section{Preliminaries}
The formalism of {\em mean field game} (MFG) \cite{lasry2007mean,huang2006large} offers a mathmetically tractable model for analysing large-scale multi-agent systems. It approximates the interactions among homogeneous agents by those between a representative agent and the population. In this section, we introduce the formulation of MFGs and some standard equilibrium concepts which we will build upon in our method.

\subsection{Mean Field Games}\label{sec:MFG}
Throughout the paper,  we focus on MFGs with finite state-action spaces and finite time horizons \cite{elie2020convergence}. 
First, consider an $N$-player game where all agents share the same local state 
space $\Smc$ and action space $\Amc$. A {\em joint state} is a tuple $(s^1,\ldots,s^N)\in \Smc^N$ where $s^i\in \Smc$ is the state of the $i$th agent. Taking the limit as $N \to \infty$, instead of modelling each agent individually, MFGs model a single representative agent and collapse the joint state into an empirical distribution $\mu \in \Pmc(\Smc)$, called a {\em mean field}, given by 
\begin{equation}
	\mu(s) \triangleq \lim_{N \to \infty} \frac{1}{N} \sum_{i=1}^N \mathds{1}_{\{s^i = s\}},
\end{equation}
where $\mathds{1}$ denotes the indicator function (i.e., $\mathds{1}_{\{x\}} = 1$ if $x$ is true and $0$ otherwise) and $\Pmc(\Smc)$ denotes the set of probability distributions over $\Smc$. The {\em transition function} $P \colon \Smc \times \Amc \times \Pmc(\Smc) \times \Smc \rightarrow [0,1]$ specifies how an agent's states evolve, i.e., an agent transits to the next state  $s_{t+1}$ with probability $P(s_{t+1} \vert s_t,a_t,\mu_t)$, depending on its current state, action, and mean field. 
Let $T \in \Nbb^+$ denote a finite time horizon. A {\em mean field flow} (MF flow for short) consists of a sequence of $T+1$ mean fields $\muv \triangleq \{\mu_t\}_{t=0}^{T}$, where the initial value $\mu_0$ is given. 
The {\em running reward} of an agent is specified by the {\em reward function} $r\colon \Smc \times \Amc \times \Pmc(\Smc) \rightarrow \Rbb$ with the exception that the reward at the last step ($t = T$) is defined separately. Following the convention, we set it as zero \citep{yang2018learning,elie2020convergence}. 
The agent's long-term reward is thus the sum 
$\sum_{t=0}^{T-1} \gamma^t r(s_t,a_t,\mu_t)$, where $\gamma \in (0,1]$ is the {\em discounted factor}. 
To summarise, an MFG is defined as a tuple $(\Smc, \Amc, P, \mu_0, r,\gamma)$. 

MFGs adopt a {\em time-varying stochastic policy} $\piv \triangleq \{\pi_t\}_{t=0}^{T}$ to characterise a strategic agent, where $\pi_t \colon \Smc \rightarrow \Pmc(\Amc)$ 
is the {\em per-step policy} at step $t$, i.e., $\pi_t$ directs the agent to choose action $a_t\sim \pi_t(\cdot\vert s_t)$. 
Given a policy $\piv$, the {\em expected return} (cumulative rewards) of an agent during the whole course of a game while interacting with an MF flow $\muv$ is given by
\begin{equation}\label{eq:exp-return}
	J(\muv,\piv) \triangleq \mathbb{E}\left[\sum_{t=0}^{T-1} \gamma^t r(s_t, a_t, \mu_t) ~\Big\vert~ s_0 \sim \mu_0, \muv, \piv, P \right].
\end{equation}

At the individual level, an agent seeks an optimal control in the form of a policy to maximise the expected return. Fixing an MF flow $\muv$, a policy $\piv$ is called a {\em best response} to $\muv$ if it maximises $J(\muv, \piv)$.
We denote the set of all best-response policies to a given $\muv$ by 
\begin{equation}\label{eq:best-response-policies}
	\Psi (\muv) \triangleq \argmax_{\piv} J(\muv, \piv).
\end{equation}
Since all agents are homogeneous, MFG prescribes that every agent uses the same policy. 
The dynamics of MF flow is thereby governed by the (discrete-time) {\em McKean-Vlasov} (MKV) equation \cite{carmona2013control}:
\begin{equation}\label{eq:MKV}
	\mu_{t+1}(s') = \sum_{s \in \Smc} \mu_t(s) \sum_{a \in \Amc} \pi_t(a \vert s)\; P(s' \vert s, a, \mu_t).
\end{equation}
Denote $\Phi(\piv)$ as the MF flow fulfilling the MKV equation above given a policy $\piv$. We say $\muv$ is {\em consistent} with $\piv$ if $\muv = \Phi(\piv)$.

At the population level, if all agents use the same policy $\piv$, the population state distribution (i.e., the mean field $\mu_t$) matches each individual's state visitation distribution (i.e., $P(\cdot \vert s_{t-1}, a_{t-1}, \mu_{t-1})$) \cite{guo2019learning}. As a result, the cumulative {\em societal rewards} (population's average rewards) coincide with an individual's expected return $J(\muv, \piv)$. Formally, let $\bar{r}(\mu, \pi)$ denote the societal reward when all agents play the same per-step policy $\pi$ under the mean field $\mu$:
\begin{equation}\label{eq:avg-reward}
	\bar{r}(\mu, \pi) \triangleq \sum_{s \in \Smc} \mu(s) \sum_{a \in \Amc} \pi(a \vert s) r(s,a,\mu).
\end{equation}
We can express $J(\muv, \piv)$ in terms of the societal reward as follows:
\begin{equation}\label{eq:societal-rewards}
	J(\muv, \piv) = \sum_{t=0}^{T-1} \gamma^t \bar{r}(\mu_t, \pi_t) \;\;\text{ if }\;\; \muv = \Phi(\piv).
\end{equation}
		
\subsection{Mean Field Nash Equilibrium}\label{sec:MFNE}

When agents are strategic (non-cooperative), the {\em mean field Nash equilibrium} (MFNE) is adopted as the solution concept, where all agents use the same best-response policy to the MF flow. Meanwhile, the MF flow is consistent with the policy.

\begin{definition}[Mean Field Nash Equilibrium]\label{def:MFE}
	A pair of MF flow and policy $(\muv^{\star}, \piv^{\star})$ is called a {\em mean field Nash equilibrium} if
	\begin{itemize}
		\item Agent rationality: $\piv^\star \in \Psi(\muv^\star)$;
		\item Population consistency: $\muv^{\star} = \Phi(\piv^{\star})$.
	\end{itemize}
\end{definition}

An MFG admits an MFNE under the standard assumptions in game theory
\cite{saldi2018markov,cui2021approximately}. The computation of MFNE typically involves a fixed-point iteration procedure for the MF flow. More formally, defining any mapping $\hat{\Psi}: \muv \mapsto \piv$ that identifies a best-response policy in $\Psi(\muv)$, we obtain the fixed-point iteration for the MF flow by alternating between $\piv = \hat{\Psi}(\muv)$ and $\muv = \Phi(\piv)$. The standard assumption for the uniqueness of MFNE is 
that the fixed-point iteration will converge to a unique MF flow.
However, it does not hold in general \cite{cui2021approximately}, which implies the existence of multiple MFNE.

\subsection{Mean Field Social Optimum}\label{sec:MFSO}
When agents are cooperative, the solution concept is the {\em mean field social optimum} (MFSO), which maximises cumulative societal rewards whilst satisfies the condition of population consistency.
\begin{definition}[Mean Field Social Optimum]\label{def:MFE}
	A pair of MF flow and policy $(\bar{\muv}^{\star},\bar{\piv}^{\star})$ is called a {\em mean field social optimum} if
	\begin{itemize}
		\item Maximum cumulative societal rewards: $J(\bar{\muv}^{\star}, \bar{\piv}^{\star}) \geq J(\muv, \piv)$ for any $(\muv, \piv)$ satisfying $\muv = \Phi(\piv)$;
		\item Population consistency: $\bar{\muv}^{\star} = \Phi(\bar{\piv}^{\star})$.
	\end{itemize}
\end{definition}
MFSO is not equivalent to MFNE. In an MFNE, $\piv^\star$ maximises the expected return $J(\muv, \piv)$ given $\muv^\star$; while an MFSO $(\bar{\muv}^\star, \bar{\piv}^\star)$ maximises $J(\muv, \piv)$ among all $(\muv, \piv)$ satisfying $\muv = \Phi(\piv)$. 
In other words, an MFSO is a particular MFNE that maximises the expected return. Note that if an MFNE exists uniquely, it is also an MFSO. Unless specified otherwise, we will use the term MFNE to denote an equilibria that is not an MFSO.
The process of finding an MFSO can be defined as a constrained optimisation problem:
\begin{equation}\label{eq:MFSO}
	\max_{\muv, \piv} J(\muv, \piv) \;\;\;\text{subject to}\;\;\; \muv = \Phi(\piv).
\end{equation}
In Sec.~\ref{sec:revisit}, we draw a connection between this optimisation problem and a particular Markov decision process (MDP).

\section{Revisiting IRL for MFG via Reducing MFG to MDP}\label{sec:revisit}

Suppose we have no access to the reward function $r(s,a,\mu)$ but have a set of expert demonstrations. {\em Inverse reinforcement learning} (IRL) aims to uncover the reward function behind these demonstrations.
In this section, we revisit the IRL method in  \cite{yang2018learning}, which showed that an MFG can be reduced to an MDP. It thus extended IRL to MFGs by applying existing single-agent IRL methods to this MDP. This reduction is also adopted in the study of reinforcement learning (RL) for MFGs \cite{carmona2019model}. 
However, we show that this reduction holds only for the {\em fully cooperative setting}, i.e., all agents share the same societal reward. As a result, IRL methods based on this reduction would lead to biased reward inferences if demonstrations are sampled from an MFNE rather than an MFSO. The exposure of this limitation motivates our restudy on IRL for MFGs. In Sec.~\ref{sec:indi}, we propose a novel IRL method for MFGs with general non-cooperative environments. 

\subsection{The Reduction from MFG to MDP}

Shown in \cite{yang2018learning,carmona2019model}, an MFG can be reduced to an MDP that describes the population's collective behaviours. The state, action and reward in this MDP corresponds to the mean field, population's collective action and societal reward in the MFG, respectively. Its dynamics coincides with the MKV equation. 
Formally, the MDP associated with an MFG $(\Smc, \Amc, P, \mu_0, r, \gamma)$ is constructed as follows:
    \begin{itemize}
        \item State: $\mu_t$, i.e., the state at step $t$ is the mean field $\mu_t$.
        \item Action: $\pi_t$, i.e., the action is a per-step policy in MFG.
        \item Reward: $\bar{r}(\mu_t, \pi_t)$, i.e., the societal reward.
        \item Deterministic Transition: $\mu_{t+1} = \Phi(\mu_t, \pi_t)$ \footnote{Here, we slightly abuse the notation $\Phi$ to denote the next mean field induced by the current mean field and current per-step policy, according to the MKV equation.} fulfilling the MKV equation defined in Eq.~\eqref{eq:MKV}.        
        \item Stationary Policy: $\varphi: \mu \mapsto \pi$.
    \end{itemize}

Intuitively, one can interpret the MDP policy $\varphi$ as a ``central controller'' who makes decisions for the overall population based on the current population state distribution. 
Let $\varphi^{\star}$ be an optimal policy for the MDP, it was claimed in \cite{yang2018learning,carmona2019model} that the MDP state-action trajectory $(\mu_0, \pi_0, \ldots, \mu_T, \pi_T)$ generated by $\varphi^{\star}$ constitutes an MFNE of the MFG, where $\pi_t = \varphi^\star(\mu_t)$ and $\mu_{t+1} = \Phi(\mu_t, \pi_t)$.

However, we show that, more precisely, the MDP state-action trajectory generated by $\varphi^{\star}$ constitutes an MFSO rather than a general MFNE. The reason lies in the construction of the MDP: the optimal MDP policy $\varphi^{\star}$ maximises the cumulative societal rewards $J(\muv, \piv) = \sum_{t=0}^{T-1}\gamma^t \bar{r}(\mu_t,\pi_t)$ and meanwhile, the deterministic transition enforces the condition of population consistency, thereby exactly solving the constrained optimisation problem of computing an MFSO as defined in Eq.~\eqref{eq:MFSO}. Intuitively, at the macroscopic level, the dynamics of population's collective behaviours is governed by an MDP only if agents are fully cooperative. 
Therefore, if an MFG is not fully cooperative or we do not assume an MFNE exists uniquely, the reduction from MFG to MDP would no longer hold.

\subsection{Population-Level IRL for MFGs}
The work \cite{yang2018learning} proposed to infer the societal reward function $\bar{r}$ by applying single-agent IRL methods to the MDP defined above. Since this method infers the population's  societal reward, we henceforth call it {\em Population-Level IRL} (PLIRL).
In PLIRL, 
we assume that the demonstrations $\Dmc = \{ (\mu^j_0, \pi^j_0, \ldots, \mu^j_T, \pi^j_T) \}_{j=1}^M$ are a set of $M$ MF flow-policy trajectories whose expectation is $(\muv^E, \piv^E)$. The goal of PLIRL is to find a suitable $\bar{r}$ that can rationalise the expert behaviour $(\muv^E, \piv^E)$. We can succinctly represent PLIRL as the following constrained optimisation problem:
\begin{equation}\label{eq:PLIRL}
\begin{aligned}
	\mathrm{PLIRL} & \left(\muv^E, \piv^E \right) =  \argmax_{\bar{r}(\mu, \pi)} \left[ J\left(\muv^E, \piv^E\right) - \max_{\muv, \piv} J(\muv, \piv) \right] \\
	&\;\;\;\;\;\;\;\;\;\;\;\;\;\;\;\; \text{ subject to }\;\; \muv = \Phi(\piv)
\end{aligned},
\end{equation}
where $J(\muv, \piv)$ is in the form of Eq.~\eqref{eq:societal-rewards}.
Intuitively, if $(\muv^E, \piv^E)$ indeed maximises the societal rewards under a feasible $\bar{r}$, then the objective would attain the maximum $0$; otherwise, it is negative. In IRL literature, problems in the form of Eq.~\eqref{eq:PLIRL} are generally solved by a bilevel optimisation procedure \cite{ziebart2008maximum,finn2016guided,fu2018learning}. In PLIRL, specifically, the upper-level task is to tune the societal reward function $\bar{r}(\mu, \pi)$ given the optimal expected return of the MDP defined above; the lower-level task is to solve an MDP based on the current reward function. In practice, $J(\muv^E, \piv^E)$ is estimated from $\Dmc$. 

Although PLIRL allows us to model and predict the population's collective behaviours, it can only handle demonstrations sampled from an MFSO because the reduction from the MFG to the MDP holds only for the fully cooperative setting. The problem setting of PLIRL does not necessarily align with the interest of each individual agent in MFGs because, in general, agents exhibit non-cooperative interactions. The equilibrium behind demonstrations is thereby more likely to be an MFNE rather than an MFSO. Applying PLIRL on MFNE demonstrations would thus lead to biased reward inferences, as is illustrated in Fig~\ref{fig:bias-example}. It can also be observed from Eq.~\eqref{eq:PLIRL}, where $J(\muv^E, \piv^E)$ is generally not the maximised societal rewards under an MFNE. As a result, the policy elicited from a biased reward function may not coincide with any MFNE induced by the ground-truth reward function, leading to an unsuitable behavioural model.  

\begin{figure}
	\centering
	\includegraphics[width=.42\textwidth]{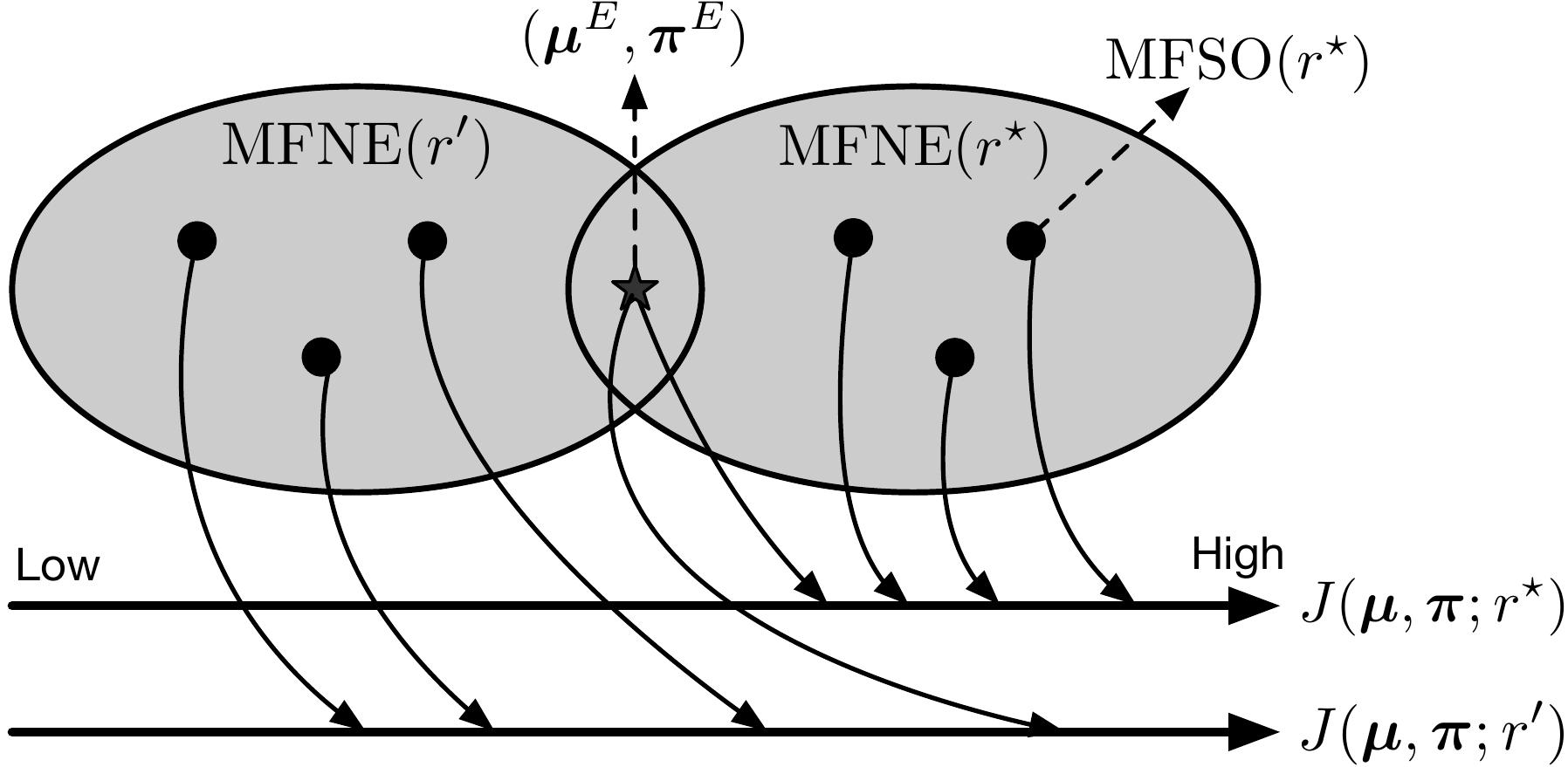}
	\caption{Illustration of the biased inference in population-level IRL. The shaded ellipse represents the set of all MFNE (including MFSO) induced by a specific reward function. The expert uses an MFNE $(\muv^E, \piv^E)$ under the ground-truth reward function $r^\star$, which is an MFSO under another reward function $r'$. Since PLIRL presupposes expert demonstrations are sampled from an MFSO, it tends to infer the societal reward induced by $r'$ rather than that induced by $r^\star$.}\label{fig:bias-example}
\end{figure}

\section{Individual-Level IRL for MFG\titlelowercase{s}}\label{sec:indi}




The discussions above justify the necessity to recover the individual reward function $r(s,a,\mu)$ from demonstrations sampled from an MFNE because the individual reward is independent of the environment and thereby allows us to model and predict both cooperative and non-cooperative behaviours in large populations. In this section, we first formalise this problem as the {\em Individual-Level IRL} (ILIRL) for MFGs, as opposed to population-level IRL. We then propose our solution framework for ILIRL with theoretical guarantees. 

\subsection{Problem Formulation}

Suppose we do not know the ground-truth reward function $r(s,a,\mu)$ but have a set of expert demonstrations $\Dmc =  \{ \tau^j \}_{j = 1}^M$ sampled from  an {\em unknown} MFNE $(\muv^E, \piv^E)$. Each $\tau^j = s^j_0,a^j_0,
\ldots, s^j_T,a^j_T$ is a state-action trajectory of an {\em individual} agent, which is sampled via $s_0 \sim \mu^E_0$, $s_t \sim P(\cdot \vert s_{t-1}, a_{t-1}, \mu^E_{t-1})$ \footnote{
Since under an MFNE, the mean field matches each individual's state visitation distribution, sampling the state from a single individual via $s_t \sim P(\cdot \vert s_{t-1}, a_{t-1}, \mu^E_{t-1})$ is equivalent to sampling that from multiple individuals, i.e., $s_t \sim \mu^E_t$.} and $a_t \sim \pi_t^E(\cdot \vert s_t)$. Following the convention in IRL literature \cite{ho2016generative,song2018multi,yu2019multi}, we assume that $\Dmc$ provides the entire supervision signals, i.e., we cannot further communicate with the expert for additional information. ILIRL for MFG asks for a reward function $r(s,a,\mu)$, under which $(\muv^E, \piv^E)$ forms an MFNE.

The individual-level inference characteristics of ILIRL are embodied in two key aspects, which distinguish it from PLIRL. First, ILIRL uses demonstrated state-action trajectories sampled from individuals. In contrast, PLIRL uses demonstrated MF flows and policies sampled from the population. Second, ILIRL aims to infer the individual reward $r(s,a,\mu)$ and we can calculate societal rewards accordingly; while PLIRL aims to infer the societal reward $\bar{r}(\mu, \pi)$ that is uninformative for acquiring individual rewards. 

To frame ILIRL as an optimisation problem, we desire an inverse operator $\mathrm{ILIRL}(\muv^E, \piv^E)$ in analogy to $\mathrm{PLIRL}$ as defined in Eq.~\eqref{eq:PLIRL}. The key idea of Eq.~\eqref{eq:PLIRL} is to choose a societal reward $\bar{r}(\mu, \pi)$ that creates a {\em margin} between the expert and every other MF flow-policy pair. Since in an MFNE the policy maximises the expected return given the MF flow, we can interpret $\mathrm{ILIRL}$ as finding a reward $r(s,a,\mu)$ that creates a margin between the expert policy $\piv^E$ and every other policy given the expert MF flow $\muv^E$:
\begin{equation}\label{eq:ILIRL}
\begin{aligned}
	\mathrm{ILIRL}&\left(\muv^E, \piv^E\right) = \argmax_{r(s,a,\mu)}  \left[ J\left(\muv^E, \piv^E \right) - \max_{\piv} J\left(\muv^E, \piv\right) \right]\\
\end{aligned},
\end{equation}
where $J(\muv, \piv)$ is in the form of Eq.~\eqref{eq:exp-return}. If $(\muv^E, \piv^E)$ is an MFNE under a valid $r(s,a,u)$, then the objective attains the maximum 0, otherwise it is negative.  Note that there may exist multiple feasible solutions to the problem. We do not intend to find all of them, as any a feasible reward function can explain expert demonstrations.


\subsection{The Mean Field IRL Framework}\label{sec:MFIRL}
We next present our proposed framework to solve the optimisation problem of ILIRL in Eq.\eqref{eq:ILIRL}, which we name as {\em Mean Field IRL} (MFIRL). The framework solves the ILIRL problem in a manner of bilevel optimisation, where the upper-level task is to tune the reward function $r(s,a,\mu)$ given the solution of the lower-level task that computes a best-response policy to the expert MF flow $\muv^E$. Recall the fixed point iteration for computing MFNE from Sec.~\ref{sec:MFNE}, where the mapping $\hat{\Psi}: \muv \mapsto \piv$ is used to identify a best-response policy to $\muv$ under a reward function $r$. Using this notation, we are left with finding a suitable $\hat{\Psi}$. But first, let us see how to solve the ILIRL problem in a more feasible way, where we estimate expected returns from $\Dmc$ and use a parameterised reward function. 
Immediately after that, the instantiation of $\hat{\Psi}$ will be given.


\begin{figure}
	\centering
	\includegraphics[width= .4\textwidth]{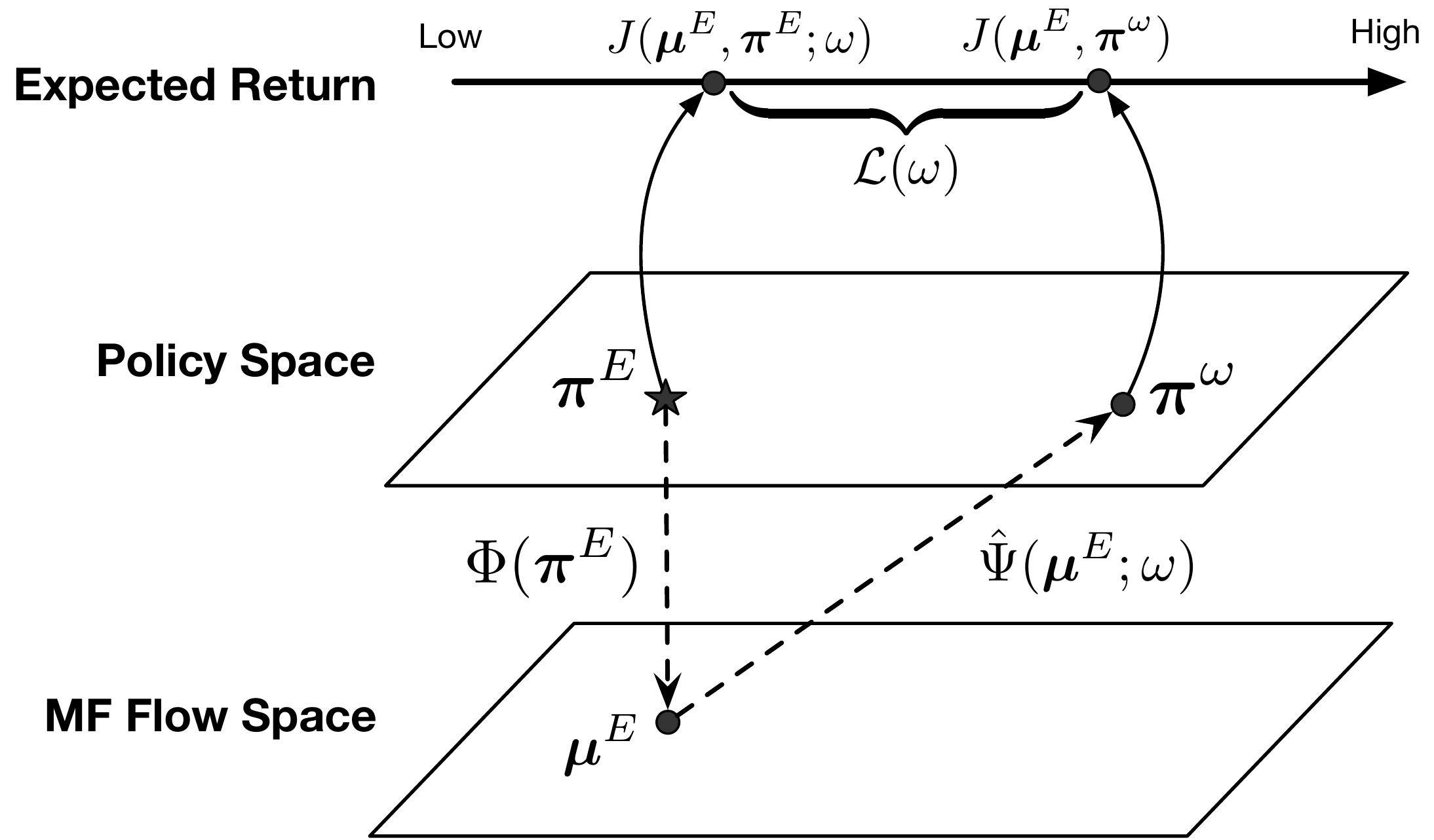}
	\caption{Illustration of the MFIRL framework. 
	The function $\hat{\Psi}(\muv^E;\omega)$ identifies a best-response policy $\piv^\omega$ to $\muv^E$ under a reward function $r_\omega$. The function $\Lmc(\omega)$ measures the difference between expected returns induced by $\piv^E$ and $\piv^\omega$ given $\muv^E$. If we can find a parameter $\omega^\star$ such that $\Lmc(\omega^\star) = 0$, then $r_\omega$ is a valid solution to ILIRL. In practice, $\muv^E$ is replaced by the empirical value $\hat{\muv}^E$ estimated from demonstrations $\Dmc$.}\label{fig:MFIRL}
\end{figure}




\subsubsection{Estimated Expected Returns and Parameterised Reward Functions.}
Since the true MF flow $\muv^E$ is unknown, we use an empirical value $\hat{\muv}^E \triangleq \{ \hat{\mu}^E_t \}_{t=0}^{T}$ estimated from $\Dmc$ by averaging the frequencies of state occurrences:
\begin{equation}\label{eq:est_mf}
\hat{\mu}^E_t(s) = \frac{1}{M} \sum_{j=1}^M  \mathds{s} \mathds{1}_{ \{ s^j_t = s \} }.
\end{equation} 
Since under an MFNE, the mean field matches each individual's state visitation distribution, $\hat{\muv}^E$ is an unbiased estimator of $\muv^E$.

Following standard practice \cite{finn2016guided,fu2018learning,yu2019multi}, we assume  the reward function $r$ is parameterised by $\omega\in \Rbb^d$ and thus write it as $r_\omega$. 
Also assume $(\muv^E, \piv^E)$ is induced by some {\em unknown} true parameter $\omega^\star$, i.e., $(\muv^E, \piv^E) = (\muv^{\omega^\star}, \piv^{\omega^\star})$. Let $\piv^\omega = \hat{\Psi}(\hat{\muv}^E; \omega)$ denote a best-response policy to $\hat{\muv}^E$ under $r_\omega$. As illustrated in Fig.~\ref{fig:MFIRL}, optimising the problem in Eq.~\eqref{eq:ILIRL} on demonstration data $\Dmc$ reduces to a search process for $\omega^\star$, where $J(\muv^E, \piv^E; \omega)$ is estimated from $\Dmc$:
\begin{equation}\label{eq:MFIRL-omega}
\begin{aligned}
	\max_{\omega} \Lmc(\omega) \triangleq  \Ebb_{\tau \sim \Dmc} \left[ \sum_{t=0}^{T-1} \gamma^t r_\omega(s_t,a_t,\hat{\mu}^E_t) \right] - J(\hat{\muv}^E, \piv^\omega).
\end{aligned}
\end{equation}

\subsubsection{Characterising Best-Response Policies}

The desired mapping $\hat{\Psi}$ relies on the {\em action value function} for MFGs, which represents the expected cumulative future rewards guided by an MF flow $\muv$ and a policy $\piv$. Formally, it is defined by
\begin{equation}\label{eq:Q}
Q(t,s,a, \muv) \triangleq r(s,a,\mu_t) + \mathbb{E} \left[ \sum_{\ell = t + 1}^T \gamma^{\ell - t} r(s_\ell, a_\ell, \mu_\ell) \Big\vert \muv, \piv, P \right].	
\end{equation}
The expected return can be expressed in terms of $Q$ as follows:
	$J(\muv, \piv) = \Ebb_{s \sim \mu_0} \left[ \sum_{a \in \Amc} \pi_0(a \vert s)\; Q(0,s,a,\muv) \right]$.
We can use {\em backward induction} to recursively compute the action value function: starting from the terminal step $t = T$: $Q(T, s, a, \muv) = r(s, a, \mu_T)$; then for $0 \leq t < T$ we recursively compute:
\begin{equation}\label{eq:Q_recursive}
		 Q(t, s, a,\muv)  =  r(s, a, \mu_t) + \gamma \mathbb{E} \big[ Q(t+1, s', a',\muv) \vert \muv, \piv, P \big].
\end{equation}

For a fixed MF flow $\muv$, we say a policy $\piv^*$ is {\em greedy} with respect to $\muv$ if $\pi^*_t(\cdot \vert s)$ picks an action 
\begin{equation}\label{eq:maxP}
	a_t \in \mathop{\argmax}_{a\in \Amc} Q(t, s, a, \muv)
\end{equation} 
{\em uniformly at random}.
Since $\piv^*$ maximises the action value function for each step, it is a best response to the corresponding MF flow $\muv$. This provides an intuition to define the mapping $\piv^\omega = \hat{\Psi}(\hat{\muv}^E; \omega)$ by letting $\piv^\omega$ be greedy with respect to $\muv^E$. More formally, we write $Q_\omega^* (t,s,a,\hat{\muv}^E)$ for the optimal action value function induced by a greedy policy given $r_\omega$. By Eq.~\eqref{eq:Q_recursive} and Eq.~\eqref{eq:maxP},  $\piv^\omega$ and $Q_\omega^*$ can be recursively computed by backward induction. The expected return induced by $\piv^\omega$ in Eq.~\eqref{eq:MFIRL-omega} can thereby be written as:
\begin{equation}\label{eq:J-omega}
	J(\hat{\muv}^E, \piv^\omega) = \Ebb_{s \sim \hat{\mu}^E_0} \left[ \sum_{a \in \Amc} \pi^\omega_0 (a \vert s)\; Q^*_\omega(0,s,a,\hat{\muv}^E) \right].
\end{equation}

\subsection{Theoretical Result}
Now, we are ready to present our main result, which shows that the optimal solution to the optimisation problem in Eq.~\eqref{eq:MFIRL-omega} is an asymptotically consistent estimator of the true reward parameter.

\begin{theorem}\label{thm:MFIRL}
Let the demonstrated trajectories in $\Dmc = \{ \tau^j \}_{j = 1}^M$ be independent and identically distributed (i.i.d.) and sampled from an MFNE induced by an unknown parameterised reward function $r_{\omega^\star}(s,a,\mu)$. With probability $1$ as the number of samples $M \to \infty$, the equation $\Lmc(\omega) = 0$ has a root $\hat{\omega}$ such that $\hat{\omega} = \omega^\star$.  
\end{theorem}
\begin{proof}
	Under the i.i.d. assumption, $\hat{\muv}^E = \muv^E$ with probability $1$ as $M \to \infty$, due to the law of large numbers. Having this, we further know that $\Lmc(\hat{\omega}) = 0$ if and only if 
	$(\muv^E, \piv^E)$ is an MFNE under reward $r_{\hat{\omega}}$. Finally, due to the fact $(\muv^E, \piv^E)$ is an MFNE induced by $r_{\omega^\star}$, there must exist one $\hat{\omega}$ such that $\hat{\omega} = \omega^\star$.
\end{proof}
\section{Practical MFIRL Algorithm}

This section develops a practical implementation of the MFIRL framework. For the sake of practical use, we consider gradient methods to optimise the objective $\Lmc(\omega)$ in Eq.~\eqref{eq:MFIRL-omega}, where the gradient $\nabla \Lmc$ is computed by 
\begin{equation}\label{eq:nablaL}
\begin{aligned}
    \nabla \Lmc & = \Ebb_{\tau \sim \Dmc} \left[ \sum_{t=0}^{T-1} \gamma^t \nabla r_\omega(s_t,a_t,\hat{\mu}^E_t) \right] - \nabla J(\hat{\muv}^E, \piv^\omega).
\end{aligned}
\end{equation}
By Eq.~\eqref{eq:J-omega}, the gradient $\nabla J(\hat{\muv}^E, \piv^\omega)$ is computed by:
\begin{equation}\label{eq:nabla-J-omega}
\begin{aligned}
	\nabla J(\hat{\muv}^E, \piv^\omega) = \Ebb_{s \sim \hat{\mu}^E_0} \Bigg[ \sum_{a \in \Amc} & Q^*_\omega(0,s,a,\hat{\muv}^E) \nabla \pi^\omega_0(a\vert s) \\
	& + \pi^\omega_0(a\vert s) \nabla Q^*_\omega(0,s,a,\hat{\muv}^E) \Bigg].
\end{aligned}
\end{equation}
While, one difficulty with such an approach is that the greedy policy $\piv^\omega$ is non-differentiable \cite{guo2019learning}. We thus need to find alternative smooth $\hat{\Psi}$ mapping. To this end, we adopt Boltzmann policy \footnote{Other smooth operators (e.g., Mellowmax \cite{asadi2017alternative}) may also be used to approximate $\hat{\Psi}$.} 
to approximate the non-differentiable greedy policy. Formally, we use $\tilde{\piv}^\omega$ to denote a Boltzmann policy and use $\tilde{Q}^*_\omega$ to represent the corresponding action value function, which are  defined by:
\begin{equation}\label{eq:policy_BOL}
	\tilde{\pi}^\omega_t (a \vert s)  \triangleq \frac{\exp \big( \beta \tilde{Q}^*_\omega (t,s,a,\hat{\muv}^E)}{\sum_{a' \in \Amc} \exp \big( \beta \tilde{Q}^*_\omega (t, s, a', \hat{\muv}^E \big)}.
\end{equation}
Here, $\beta > 0$ is the {\em inverse Boltzmann temperature} controlling the degree of approximation. Note that we recover the optimality (i.e., the greedy policy) if $\beta \to \infty$. By Eq.~\eqref{eq:policy_BOL} and Eq.~(\ref{eq:Q_recursive}), gradients $\nabla \tilde{\pi}^\omega_t (a \vert s)$ and $\nabla \tilde{Q}_{\omega}^*(t,s,a,\hat{\muv}^E)$ can be recursively calculated:

\begin{equation}\label{eq:gradient_policy}
        \begin{aligned}
          \nabla \tilde{\pi}^\omega_t (a \vert s) 
         =&~ \tilde{\pi}^\omega_t (a \vert s) \cdot \beta \bigg[\nabla \tilde{Q}_{\omega}^*(t,s,a,\hat{\muv}^E) \\
          &\;\;\;\;\;\;\;\;\;\;\;\;\;\;- \Ebb_{a' \sim \tilde{\pi}^\omega_t (\cdot \vert s)} \left[ \nabla \tilde{Q}_{\omega}^*(t,s,a',\hat{\muv}^E) \right] \bigg],
         \end{aligned}
\end{equation}

\begin{equation}\label{eq:gradient_Q}
 \begin{aligned}
 	 \nabla \tilde{Q}_{\omega}^*&(t,s,a,\hat{\muv}^E) = \nabla r_{\omega}(s,a,\hat{\mu}^E_t)\\
 	& + \gamma \Ebb_{s' \sim P} \Bigg[ \sum_{a' \in \Amc}  \tilde{Q}_{\omega}^*(t + 1,s',a',\hat{\muv}^E) \nabla \tilde{\pi}^\omega_{t+1} (a' \vert s') \\
	& \;\;\;\;\;\;\;\;\;\;\;\;\;\; + \Ebb_{a' \sim \tilde{\pi}^\omega_{t+1} (\cdot \vert s')} \left[ \nabla \tilde{Q}_{\omega}^*(t+1,s',a',\hat{\muv}^E) \right]  \Bigg].
	\end{aligned}
\end{equation}


Detailed derivation of $\nabla \tilde{\pi}^\omega_t (a \vert s)$ in Eq.~\eqref{eq:gradient_policy} is given in Appendix~A. 
Substituting $\nabla \tilde{Q}_{\omega}^*$ for $\nabla Q_{\omega}^*$ in Eq.~\eqref{eq:nabla-J-omega} yields an approximate $\nabla \Lmc$, which we denote by $\nabla \tilde{\Lmc}$.

\begin{algorithm}
   \caption{Practical MFIRL Algorithm}\label{alg:AL-MFIRL}
\begin{algorithmic}[1]
   \STATE {\bf Input:} MFG with parameters $(\Smc, \Amc, P, \mu_0, \gamma)$ and expert demonstrations $\Dmc = \{ \tau^j \}_{j = 1}^M$.
   \STATE {\bf Initialisation:} Initialise reward parameter $\omega$.
   \STATE Estimate empirical expert MF flow $\hat{\muv}^E$ according to Eq.~\eqref{eq:est_mf}.
   \FOR{each epoch}
   		\FOR{$t = T, \ldots, 0$}
   		\STATE Calculate $\nabla \tilde{\pi}^\omega_t$ and $\nabla \tilde{Q}_{\omega}^*$ according to Eq.~\eqref{eq:gradient_policy} and Eq.~\eqref{eq:gradient_Q}. 
   		\ENDFOR  
   		\STATE Calculate the empirical gradient $\hat{\nabla} \tilde{\Lmc}$ according to Eq.~\eqref{eq:nablaL}.
   		\STATE  Update $\omega$ to increase $\Lmc$ according to  $\hat{\nabla} \tilde{\Lmc}$.
   \ENDFOR
   \STATE {\bfseries Output:} Learned reward function $r_{\omega}$.
\end{algorithmic}
\end{algorithm}

\subsubsection*{Monte-Carlo Simulation under Unknown Dynamics}
The transition function $P$ is generally unknown in practice. We can instead use Monte-Carlo simulation to estimate empirical gradients $\hat{\nabla} \tilde{Q}_{\omega}^*$ and $\hat{\nabla} \tilde{\pi}^\omega_t$, that is, suppose we have the access to a simulator of the environment and calculate $\hat{\nabla} \tilde{Q}_{\omega}^*$ by estimating the expectation with respect to $s'$ in Eq.~\eqref{eq:gradient_Q} using sampled states $s'$ from the simulator. We denote the resulting empirical gradient of $\Lmc$ by $\hat{\nabla}\tilde{\Lmc}$.

\begin{figure*}[!htp]
	\centering
	\includegraphics[width=\textwidth]{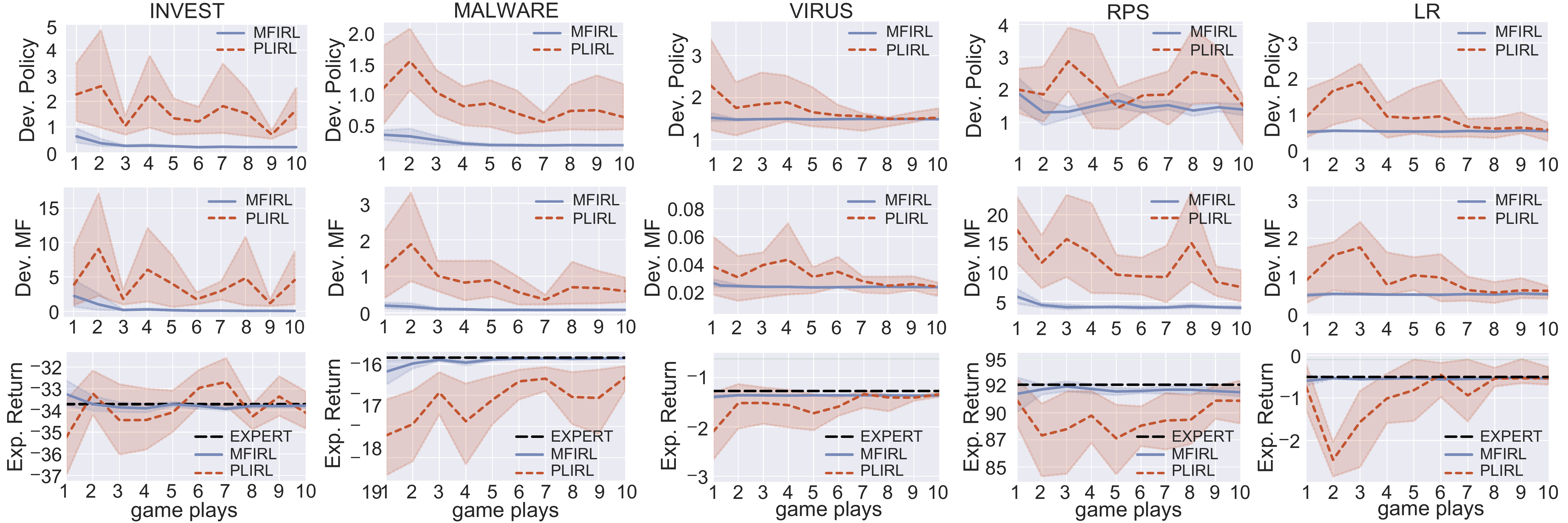}
	\caption{Results for numerical tasks under original dynamics. The solid line shows the median and the shaded area represents the standard deviation over 10 independent runs. }\label{fig:original}
\end{figure*}

\subsubsection*{Truncated Recursive Computation of Gradients}
With the horizon of an MFG increasing, the recursive computation of $\nabla \tilde{Q}^*_\omega(t,s,a,\hat{\muv}^E)$ tends to be intractable. To overcome this potential issue in practice, we can approximate $\nabla \tilde{Q}^*_\omega(t,s,a,\hat{\muv}^E)$ by truncating the recursion. More formally, let $\omega^{(i)}$ denote the reward parameter in the $i$th round of update. In the $i+1$th update, we look $H$-step ($H<T$) ahead according to Eq.~\eqref{eq:gradient_Q} until we arrive at time step $t+H$, where we stop tracing deeper by replacing $\nabla \tilde{Q}^*_\omega(t+H+1,s,a,\hat{\muv}^E) \vert_{\omega = \omega^{(i)}}$ and $\nabla \tilde{\pi}^\omega_{t+H+1} (a' \vert s')\vert_{\omega = \omega^{(i)}}$ with the corresponding variables induced by the old value $\omega^{(i-1)}$. Using old parameter values to approximate the corresponding variables induced by new values has also been adopted in the RL literature \cite{schulman2015trust,schulman2017proximal}. The intuition is that the change of the parameter between two updates is not severe. Meanwhile, due to the presence of the discount factor, this approximation would not cause a significant error in the estimation of gradients. 

To summarise, we present the pseudocode in Alg.~\ref{alg:AL-MFIRL}.


\section{Experiments}\label{sec:experiments}

In experiments, we seek to answer the following key question: {\em Can MFIRL efficiently recover the underlying individual reward function, regardless of whether the environment is cooperative or not?}  To this end, 
we evaluate the quality of a learned reward function $r_\omega$ by comparing its induced MFSO $(\bar{\muv}^\omega, \bar{\piv}^\omega)$ and the ground-truth MFSO $(\bar{\muv}^\star, \bar{\piv}^\star)$, because an MFSO is more likely to be unique than an MFNE in terms of policy and it is always unique in term of expect return. Specifically, we use the following metrics:
\begin{enumerate}
	\item {\em Policy Deviation} (Dev. Policy). We use the cumulative KL-divergence,  $\sum_{t = 0}^{T} \sum_{s \in \Smc} D_{\KL} \big( \bar{\pi}_t^\star(\cdot \vert s) \parallel \bar{\pi}_t^\omega(\cdot \vert s) \big)$, to measure the difference over two policies.
    \item {\em MF flow Deviation} (Dev. MF). Similarly, we use the cumulative KL-divergence, $\sum_{t = 0}^{T} D_{\KL} \big(  \bar{\mu}_t^\star \parallel \bar{\mu}_t^\omega \big)$, to measure the difference over two MF flows.
    \item {\em Expected return.} (Exp. Return) The expected return of $(\bar{\muv}^\omega, \bar{\piv}^\omega)$ and $(\bar{\muv}^\star, \bar{\piv}^\star)$ under the ground-truth reward function. 
\end{enumerate} 

We compare MFIRL against the PLIRL in \cite{yang2018learning} as it is the only IRL method for MFGs in the literature as of the present. Also, existing multi-agent IRL and imitation learning methods (e.g., MA-GAIL \cite{song2018multi} and MA-AIRL \cite{yu2019multi}) scale poorly when the population size is in hundreds, due to the exponential growth of joint state-action spaces. 
We carry out two classes of tests, one on numerical MFG models 
and the other on simulated mixed cooperative-competitive battle games.
For both classes of tests, we set the discounted factor $\gamma=0.99$ and the time horizon as $50$, which is the same as the number of time steps used in \cite{song2018multi,yu2019multi}. For MFIRL, we set the inverse Boltzmann temperature $\beta = 1$.
In each task, we set $N$ agents where $N$ is a large but finite number. We sample state-action trajectories from each of $N$ individual agents by executing a pre-trained expert policy. We call an execution of the expert policy on all $N$ agents a {\em game play}. MFIRL directly takes as input these individual trajectories. While, each demonstrated expert policy and mean field fed into PLIRL are estimated by averaging occurrence frequencies of states and actions in all $N$ trajectories per game play.
For both algorithms, we adopt the same neural network architecture as the reward model: two hidden layers of $64$ leaky rectified linear units ($\mathtt{ReLU}$) each. 
Implementation details are given in Appendix~C. 

\subsection{Numerical Models}

\subsubsection*{Settings.} We evaluate MFIRL on five numerical discrete MFG models: {\em investment in  product quality} \cite{weintraub2010computational,subramanian2019reinforcement} ({\tt INVEST} for short), {\em malware spread} \cite{huang2016mean,huang2017mean,subramanian2019reinforcement} ({\tt MALWARE}), {\em virus infection} \cite{cui2021approximately} ({\tt VIRUS}), {\em Rock-Paper-Scissors} \cite{cui2021approximately} ({\tt RPS}) and {\em Left-Right} \cite{cui2021approximately} ({\tt LR}), ordered in decreasing complexity. These models simulate a series of large-scale decision making scenarios in the real world. Among these five models, we use {\tt VIRUS} and {\tt LR} as cooperative scenarios, i.e., the demonstrations in these two models are sampled from MFSO. The remaining three models are used as non-cooperative scenarios. Statistical information of these models is summarised in Tab.~\ref{tab:numerical}. Detailed descriptions and settings can be found in Appendix~B. We take $100$ agents for each model.
We train MFNE experts through the fixed-point iteration as described in Sec.~\ref{sec:MFNE}, and train MFSO experts using DDPG \cite{lillicrap2015continuous} to solve the MDP reduced from MFG. 

\begin{table}
\caption{Statistics of Numerical MFG Models}\label{tab:numerical}
\centering
	\begin{tabular}{l c c c c}
		\toprule
		Model & States & Actions &  Cooperative \\
		\midrule
		{\tt INVEST} & 10 & 2 & \xmark \\
		{\tt MALWARE} & 10 & 2 &\xmark\\
		{\tt VIRUS} & 2 & 2 & \cmark\\
		{\tt RPS} & 3 & 3 &  \xmark\\
		{\tt LR} & 3 & 2 & \cmark\\ 
		\bottomrule
	\end{tabular}
\end{table}

\begin{table*}[!h]
\small
    \caption{Results for numerical models under new dynamics. Mean and variance are taken across 10 independent runs.}
    \label{tab:new}
    \centering 
    \begin{tabular}{c  c  c  c  c  c  c}
        \toprule
        \multirow{3}{*}{Metric} & \multirow{3}{*}{Algorithm} & \multicolumn{5}{c}{Task}\\
        \cmidrule(r){3-7}
        & & {\tt INVEST} & {\tt MALWARE} & {\tt VIRUS} & {\tt RPS} & {\tt LR}\\
        \midrule
        \multirow{2}{*}{Dev. Policy} & MFIRL &  \textbf{0.305} $\pm$ 0.017 & \textbf{0.411} $\pm$ 0.025 & \textbf{1.544} $\pm$ 0.012 &  \textbf{7.089} $\pm$ 0.541 & \textbf{0.683} $\pm$ 0.035\\
        & PLIRL &  1.130 $\pm$ 0.334 & 1.466 $\pm$ 1.322 & 1.892 $\pm$ 0.237 & 7.550 $\pm$ 0.841 & 0.734 $\pm$ 0.373 \\
        \midrule
        \multirow{2}{*}{Dev. MF} & MFIRL & \textbf{0.464} $\pm$ 0.029 & \textbf{0.435} $\pm$ 0.007 & \textbf{0.057} $\pm$ 0.0004 &  \textbf{2.932} $\pm$ 0.057 & 0.353 $\pm$ 0.029\\ 
         & PLIRL  & 1.510 $\pm$ 0.697 & 1.731 $\pm$ 2.207 & 0.076 $\pm$ 0.0085 &  3.112 $\pm$ 0.569 & \textbf{0.348} $\pm$ 0.310 \\
        \midrule
        \multirow{3}{*}{Expected Return}
        & EXPERT & -35.051 & -18.055 & -1.167 & 94.274 & -0.518\\
        & MFIRL & \textbf{-35.542} $\pm$ 0.677 & \textbf{-18.519} $\pm$ 0.245 & -1.614 $\pm$ 0.042 & \textbf{93.578} $\pm$ 2.508 & -0.563 $\pm$ 0.078 \\
        & PLIRL & -35.917 $\pm$ 2.548 & -19.151 $\pm$ 0.507 & \textbf{-1.553} $\pm$ 0.170 & 93.212 $\pm$ 0.493 & \textbf{-0.547} $\pm$ 1.080 \\
        \bottomrule
    \end{tabular}
\end{table*}

\begin{figure*}[!h]
    \centering
    \includegraphics[width =.77\textwidth]{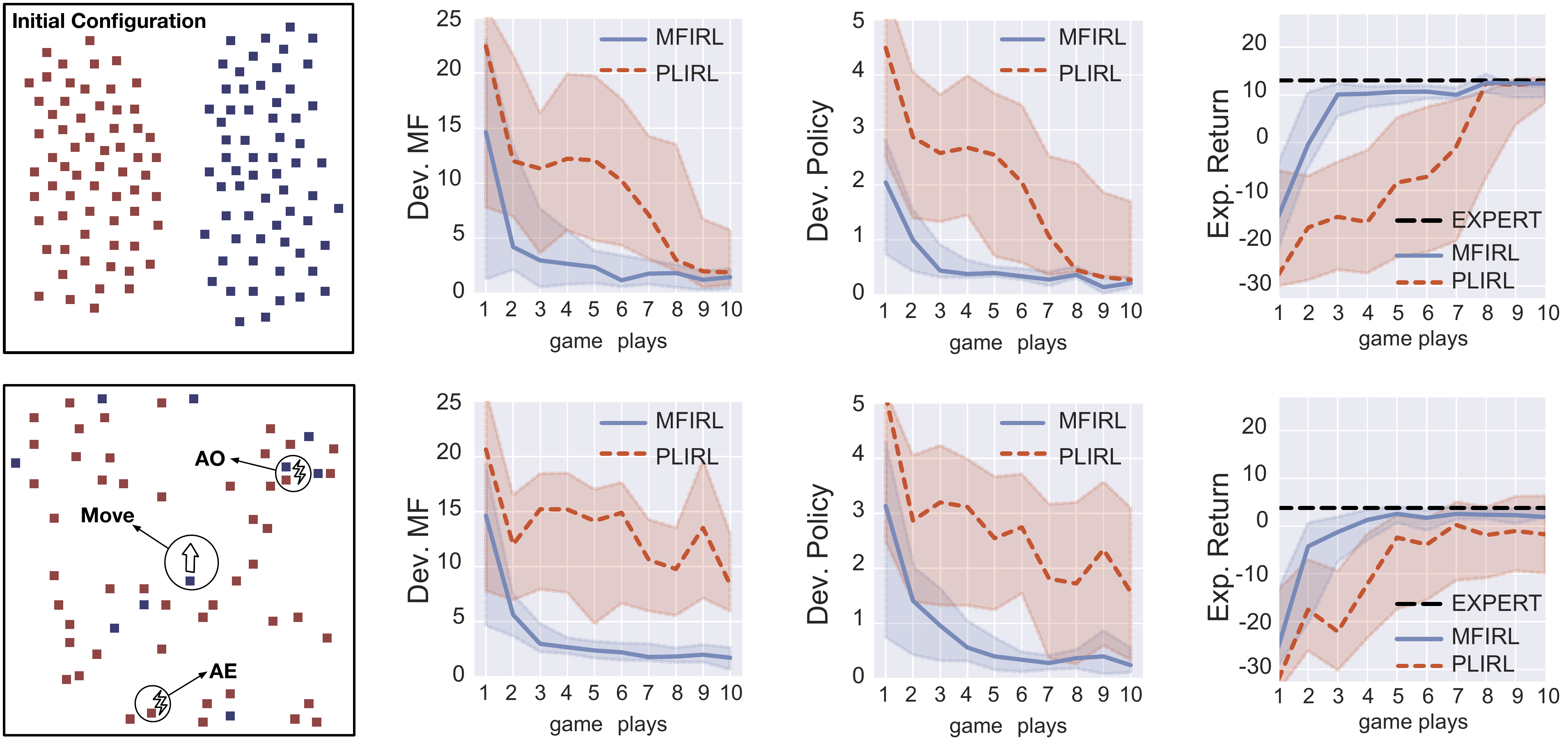}
    \caption{Illustration and results for simulated battle games. The first and second rows correspond to the cooperative and mixed cooperative-competitive setting, respectively. Results are averaged over 10 independent runs.}
    \label{fig:game}
\end{figure*}

\subsubsection*{Reward Recovery under Original Dynamics} 
We first conduct tests under fixed environment dynamics. Fig.~\ref{fig:original} depicts results. On all non-cooperative models, MFIRL achieves near-expert performance, while PLIRL shows larger deviations. The reason is that the demonstrations in these models are sampled from an MFNE rather than MFSO. This verifies the biased reward inference in PLIRL. On cooperative models, both MFIRL and PLIRL show expert-like performance, but MFIRL is more sample efficient. This is because one game play provides $N$ samples to our MFIRL but can only provide one sample of MF flow and policy to PLIRL. 


\subsubsection*{Robustness to New Dynamics.}
To investigate the robustness of the learned reward function $r_\omega$ in the face of uncertainties in dynamics, we change the transition function (see Appendix~B for details), recompute MFSO induced by the ground-truth reward and $r_\omega$ (trained with 10 game plays), respectively. Results are summarised in Tab.~\ref{tab:new}. MFIRL has comparable performance with PLIRL on cooperative models but shows much smaller errors on non-cooperative models. We attribute the low robustness of PLIRL to the conjecture that the changing dynamics can exacerbate the biased reward inference. 

\subsection{Simulated Battle Games}

\subsubsection*{Settings} The {\em Mixed Cooperative-Competitive Battle Game} \cite{zheng2018magent} contains two groups of homogenous agents fighting against each other in a 2D grid. The goal of each group is to destroy all opponents in the other group. Each group has $64$ agents. Each agent can move to the neighbourhood or take attack actions. The setting is illustrated in Fig.~\ref{fig:game}. For each agent, the default reward setting is: $-0.005$ for every move (M), $0.2$ for attacking an opponent (AO), $5$ for killing an opponent (K), $-0.1$ for attacking an empty grid (AE). If an agent is attacked or destroyed, it also receives a reward $-1$. To adapt the game for MFG, we supply an additional term to the default reward function according to the following intuition: it is less risky to attack a foe if more friends are nearby. The mean field can thus be involved in the reward to measure the average distance between an agent and all other friends. Formally, let $r(s,a)$ denote the default reward and write modified reward function by:
\begin{equation*}
	r(s,a,\mu) = r(s,a) - d \cdot \mathds{1}_{\{a = \text{AO}\}} \cdot \Ebb_{s' \sim \mu} \left[ \mathtt{dist}(s,s') \right],
\end{equation*}
where $\mathtt{dist}$ denotes the Manhattan distance and $d>0$ controls its importance. 
In experiments, we set $d = 0.1$.
We generate two types of environments: (1) Mixed cooperative-competitive: the reward of each agent at each step is defined above. (2) Fully cooperative: the reward of each agent is set as the societal reward of the group it belongs to.
We train experts (of each group) using MADDPG \cite{lowe2017multi}, a multi-agent actor-critic algorithm. If an agent is destroyed halfway, we treat all variables of it as null in the subsequent training. 

\subsubsection*{Results.} Results are reported in Fig.~\ref{fig:game}. Consistently, MFIRL demonstrates higher accuracy in non-cooperative environments and higher sample efficiency in both cooperative and non-cooperative environments. PLIRL again shows large deviations from the expert performance. 
To summarise, MFIRL can accurately recover individual reward functions for MFGs in the non-cooperative setting with high sample efficiency, in line with our theoretical analysis.  
\section{Conclusions and Future Work}
This paper amounts to an effort towards individual-level IRL for MFGs. We reveal that the reduction from MFG to MDP holds only for the fully cooperative setting, which restricts the suitability of existing IRL methods in general MFGs. 
In order to handle MFGs with general non-cooperative environments, we propose and formalise the individual-level IRL problem that asks for recovering an individual reward function for MFGs. To address this problem, we propose MFIRL, the first dedicated IRL framework for MFGs that can deal with both cooperative and non-cooperative environments. Moreover, by making a series of approximations to the MFIRL framework, we develop a practical algorithm effective for MFGs with unknown dynamics. Experiments on both cooperative and non-cooperative scenarios verify the advantages of MFIRL on reward recovery, sample efficient and robustness to changing dynamics. 

Alongside the direction opened up by this work, a straightforward future work is to scale MFIRL to MFGs with continuous or high-dimensional state-action spaces. Another promising work is to study IRL methods for general MFGs that can tolerate imperfect expert demonstrations. A third future work is to apply MFIRL to more real-world scenarios, e.g., dynamic demand management in power grids \cite{bagagiolo2014mean} and behaviour analysis in large social media \cite{yang2018learning}.




\newpage
\balance
\bibliographystyle{ACM-Reference-Format} 
\bibliography{main}

\newpage
\onecolumn

\setcounter{section}{0}
\setcounter{equation}{0}
\renewcommand{\theequation}{\thesection\arabic{equation}}
\newtheorem{lem}{Lemma}[section]
\newtheorem{cor}{Corollary}[section]
\newtheorem{prop}{Proposition}[section]
\newcommand{\boltz}{B_{\beta}}

\appendix

\begin{center}
	{\LARGE\bf Appendices}
\end{center}

\section{Derivation of Gradients}\label{app:gradients}

We present the detailed derivation of $\nabla \tilde{\pi}^\omega_t (a \vert s)$ below.

\begin{equation*}
\begin{aligned}
    \nabla \tilde{\pi}^\omega_t (a \vert s) & = \tilde{\pi}^\omega_t (a \vert s) \nabla \ln \tilde{\pi}^\omega_t (a \vert s) \\
    & = \tilde{\pi}^\omega_t (a \vert s) \nabla \ln \left( \frac{\exp\left(\beta \tilde{Q}_\omega^*\left(t,s,a,\hat{\muv}^E\right)\right)}{\sum_{a' \in \Amc} \exp\left(\beta \tilde{Q}_\omega^*\left(t,s,a',\hat{\muv}^E\right)\right)  }  \right)\\
    & = \tilde{\pi}^\omega_t (a \vert s) \nabla \left[ \beta \tilde{Q}_\omega^*\left(t,s,a,\hat{\muv}^E\right) - \ln \sum_{a' \in \Amc} \exp \left( \beta \tilde{Q}_\omega^*\left(t,s,a',\hat{\muv}^E\right) \right) \right]\\
    & = \tilde{\pi}^\omega_t (a \vert s)  \left[ \beta \nabla \tilde{Q}_\omega^*\left(t,s,a,\hat{\muv}^E\right) - \frac{\nabla \sum_{a' \in \Amc} \exp \left( \beta \tilde{Q}_\omega^*\left(t,s,a',\hat{\muv}^E\right) \right)}{\sum_{a'' \in \Amc} \exp \left( \beta \tilde{Q}_\omega^*\left(t,s,a'',\hat{\muv}^E\right) \right) } \right]\\
    & = \tilde{\pi}^\omega_t (a \vert s) \left[ \beta \nabla \tilde{Q}_\omega^*\left(t,s,a,\hat{\muv}^E\right) - \sum_{a' \in \Amc} \frac{\exp\left( \beta \tilde{Q}_\omega^*(t,s,a',\hat{\muv}^E) \right)\cdot \beta \cdot \nabla \tilde{Q}_\omega^*(t,s,a',\hat{\muv}^E)}{\sum_{a'' \in \Amc} \exp \left( \beta \tilde{Q}_\omega^*\left(t,s,a'',\hat{\muv}^E\right) \right)} \right]\\
    & = \tilde{\pi}^\omega_t (a \vert s) \left[ \beta \nabla \tilde{Q}_\omega^*\left(t,s,a,\hat{\muv}^E\right) - \sum_{a' \in \Amc} \tilde{\pi}_t^\omega(a' \vert s) \cdot \beta \cdot \nabla \tilde{Q}_\omega^*(t,s,a',\hat{\muv}^E)  \right]\\
    & = \tilde{\pi}^\omega_t (a \vert s) \cdot \beta \bigg[\nabla \tilde{Q}_{\omega}^*(t,s,a,\hat{\muv}^E)  - \sum_{a' \in \Amc} \tilde{\pi}^\omega_t (a' \vert s)  \nabla \tilde{Q}_{\omega}^*(t,s,a',\hat{\muv}^E) \bigg]\\
    & = \tilde{\pi}^\omega_t (a \vert s) \cdot \beta \bigg[\nabla \tilde{Q}_{\omega}^*(t,s,a,\hat{\muv}^E) 
          - \Ebb_{a' \sim \tilde{\pi}^\omega_t (\cdot \vert s)} \left[ \nabla \tilde{Q}_{\omega}^*(t,s,a',\hat{\muv}^E) \right] \bigg].
\end{aligned}
\end{equation*}

\section{Numerical Model Descriptions}\label{app:task}

\subsection{Investment in Product Quality}

{\bf Model.} This model is adapted from \cite{weintraub2010computational} and \cite{subramanian2019reinforcement} that captures the investment decisions in a fragmented market with a large number of firms. Each firm produces the same kind of product. The state of a firm $s \in \Smc = \{ 0,1, \ldots, 9\}$ denotes the product quality. At each step, each firm decides whether or not to invest in improving the quality of the product. Thus the action space is $\Amc = \{0, 1\}$. When a firm decides to invest, the quality of the product manufactured by it increases uniformly at random from its current value to the maximum value 9, if the average mean field for that product is below a particular threshold $q$. If this average mean field value is above $q$, then the product quality gets only half of the improvement as compared to the former case. This implies that when the average quality in the economy is below $q$, it is easier for each firm to improve its quality. When a firm does not invest, its product quality remains unchanged. Formally, the dynamics is given by:

\begin{equation*}
	s_{t+1} = \left\{
	\begin{aligned}
		 & s_{t} + \lfloor \chi_t  ( 10- s_{t} ) \rfloor, \text{~~if~~}\langle \mu_{t} \rangle < q \text{~~and~~} a_{t} = 1\\
		 & s_{t} + \lfloor \chi_t  ( 10- s_{t} ) /2 \rfloor, \text{~~if~~}\langle \mu_{t} \rangle \geq q\text{~~and~~}a_{t} = 1\\
		 & s_{t}, \text{~~if~~}a_{t} = 0
	\end{aligned}\right..
\end{equation*} 

An agent incurs a cost due to its investment and earns a positive reward due to its own product quality and a negative reward due to the average product quality, which we denote by $\langle \mu_t \rangle \triangleq \sum_{s \in \Smc} s \cdot \mu_t(s)$. The final reward is given as:
\begin{equation*}
	r(s_t,a_t,\mu_t) =  d \cdot s_{t} / 10 - c \cdot \langle \mu_{t} \rangle - \alpha \cdot a_{t} 
\end{equation*}
\smallskip

\noindent{\bf Settings.}  We set $d=0.3$, $c=0.2$, $\alpha=0.2$ and probability density $f$ for $\chi_t$ as $U(0,1)$. We set the threshold $q$ to $4$ and $5$ for the original and new environments, respectively. The initial mean field $\mu_0$ is set as a uniform distribution, i.e, $\mu_0(s) = 1 / |\Smc|$ for all $s \in \Smc$.

\subsection{Malware Spread}

{\bf Model.}  The malware spread model is presented in \cite{huang2016mean,huang2017mean} and used as a numerical study for MFG in \cite{subramanian2019reinforcement}. This model is representative of several problems with positive externalities, such as flu vaccination and economic models involving the entry and exit of firms. 
Here, we present a discrete version of this problem:  Let $\Smc = \{0, 1, \ldots, 9\}$ denote the state space (level of infection), where $s = 0$ is the most healthy state and $s = 9$ is the least healthy state. The action space $\Amc = \{0, 1\}$, where $a = 0$ means $\mathtt{Do Nothing}$ and $a = 1$ means $\mathtt{Intervene}$. The dynamics is given by
\begin{equation*}
	s_{t+1} = \left\{
	\begin{aligned}
		& s_t + \lfloor \chi_t  ( 10- s_t ) \rfloor,  \text{~~if~~}  a_t = 0\\
		& 0, \text{~~if~~}a_t = 1
	\end{aligned}\right.,
\end{equation*}
where $\{\chi_t\}_{0\leq t \leq T}$ is a $[0, 1]$-valued i.i.d. process with probability density $f$. The above dynamics means the $\mathtt{Do Nothing}$ action makes the state deteriorate to a worse condition; while the $\mathtt{Intervene}$ action resets the state to the most healthy level. Rewards are coupled through the average mean field, i.e., $\langle \mu_{t} \rangle$. An agent incurs a cost $(k + \langle \mu_t \rangle) s_t$, which captures the risk of getting infected, and an additional cost of $\alpha$ for performing the $\mathtt{Intervene}$ action. The reward sums over all negative costs:
\begin{equation*}
	r(s_t, a_t, \mu_t) = -(k + \langle \mu_t \rangle)s_t/10 - \alpha \cdot a_t.
\end{equation*}
\smallskip

\noindent{\bf Settings.} Following \cite{subramanian2019reinforcement}, we set $k = 0.2$, $\alpha = 0.5$, and the probability density $f$ to the uniform distribution $U(0,1)$ for the original dynamics. We change the density $f$ to uniform distribution $U(0.5,1)$ for new dynamics. The initial mean field $\mu_0$ is set as a uniform distribution.

\subsection{Virus Infection} 

{\bf Model.} This is a virus infection used as a case study in \cite{cui2021approximately}. There is a large number of agents in a building. Each of them can choose between ``social distancing'' ($D$) or ``going out'' ($U$). If a ``susceptible'' ($S$) agent chooses social distancing, they may not become ``infected'' ($I$). Otherwise, an agent may become infected with a probability proportional to the number of agents being infected. If infected, an agent will recover with a fixed chance every time step. Both social distancing and being infected have an associated negative reward.
Formally, let $\Smc = \{S,I\}, \Amc = \{U,D\}, r(s,a, \mu_t) = -\mathds{1}_{\{s = I\}} - 0.5 \cdot \mathds{1}_{\{s = D\}}$. The transition probability is given by
\begin{equation*}
	\begin{aligned}
		P(s_{t+1} = S \vert s_t = I, \cdot, \cdot) & = 0.3\\
		P(s_{t+1} = I \vert s_t = S, a_t = U, \mu_t) & = 0.9^2 \cdot \mu_t(I)\\
		P(s_{t+1} = I \vert s_t = S, a_t = D, \cdot) &= 0.
	\end{aligned}
\end{equation*} 
\smallskip

\noindent{\bf Settings.} The initial mean field $\mu_0$ is set as a uniform distribution. We modify the transition function for the new dynamics as follows:
\begin{equation*}
	\begin{aligned}
		P(s_{t+1} = S \vert s_t = I, \cdot, \cdot) & = 0.3\\
		P(s_{t+1} = I \vert s_t = S, a_t = U, \mu_t) & = 0.8^2 \cdot \mu_t(I)\\
		P(s_{t+1} = I \vert s_t = S, a_t = D, \cdot) &= 0.
	\end{aligned}
\end{equation*}

\subsection{Rock-Paper-Scissors} This model is adapted by \cite{cui2021approximately} from the generalized non-zero-sum version of {\em Rock-Paper-Scissors} game. 
Each agent can choose between ``rock'' ($R$), ``paper'' ($P$) and ``scissors'' ($S$), and obtains a reward proportional to double the number of beaten agents minus the number of agents beating the agent. Formally, let $\Smc = \Amc = \{R, P, S\}$, and for any $a \in \Amc, \mu_t \in \Pmc(\Smc)$:
\begin{equation*}
	\begin{aligned}
		r(R, a, \mu_t) &= 2 \cdot \mu_t(S) - 1\cdot \mu_t(P),\\
		r(P, a, \mu_t) &= 4 \cdot \mu_t(R) - 2 \cdot \mu_t(S),\\
		r(S, a, \mu_t) &= 6 \cdot \mu_t(P) - 3 \cdot \mu_t(R).
	\end{aligned}
\end{equation*}

The transition function is deterministic: $p(s_{t+1} \vert s_t, a_t, \mu_t) = \mathds{1}_{\{s_{t+1} = a_t\}}.$
\smallskip

\noindent{\bf Settings.} The initial mean field $\mu_0$ is set as a uniform distribution. Same to the setting in Left-Right, for new dynamics, we add randomness to the transition function such that with probability 0.2 picking next state arbitrarily.

\subsection{Left-Right}

{\bf Model.} This model is used in \cite{cui2021approximately}. There is a group of agents making sequential decisions to moving ``left'' or ``right''. At each step, each agent is at a position (state) either ``left'', ``right'' or ``center'', and can choose to move either ``left'' or ``right'', receives a reward according the current population density (mean field) at each position, and with probability one (dynamics) they reach ``left'' or ``right''. Once an agent leaves ``center'', she can never head back and can only be in left or right thereafter. Formally, we configure the MFG as follows: $\Smc = \{C, L, R\}$, $\Amc = \Smc \setminus \{C\}$, the reward $$r(s,a,\mu_t) = -\mathds{1}_{\{s = L\}}\cdot \mu_t(L) -\mathds{1}_{\{s = R\}}\cdot \mu_t(R).$$ This reward setting means each agent will incur a negative reward determined by the population density at her current position. The transition function is deterministic that directs an agent to the next state with probability one: $P(s_{t+1} \vert s_t, a_t, \mu_t) = \mathds{1}_{\{s_{t+1} = a_t\}}.$
\smallskip

\noindent{\bf Settings.} The initial mean field $\mu_0$ is set as $\mu_0(L) = \mu_0(R) = 0.5$. For new dynamics, we add the randomness to the transition function. With probability $0.8$, the agent moves to the state determined by the action, and with probability $0.2$, the agent randomly moves to either ``left'' or ``right''.

\section{Detailed Experiment Settings}\label{app:exp}

{\bf Feature representations.} We use one-hot encoding to represent states and actions. Let $\{1,2,\ldots, |\Smc|\}$ denote an enumeration of $\Smc$ and $\left[s_{[1]}, s_{[2]}, \ldots, s_{[|\Smc|]}\right]$  denote a vector of length $|\Smc|$, where each component stands for a state in $\Smc$. The state $j$ is denoted by $\big[0, \ldots, 0, s_{[j]} = 1, 0,$ $\ldots, 0 \big]$. An action is represented through the same manner. A mean field $\mu$ is represented by a vector $\left[\mu(s_{[1]}), \mu(s_{[2]}), \ldots, \mu(s_{[|\Smc|]})\right]$. 
\smallskip

\noindent{\bf Reward Models.} The reward mode $r_\omega$ takes as input the concatenation of feature vectors of $s$, $a$ and $\mu$ and outputs a scalar as the reward. We adopt the neural network (a four-layer perceptron) with the Adam optimiser and the Leaky ReLU activation function. The sizes of the two hidden layers are both 64. The learning rate is $10^{-4}$.
\smallskip

\noindent{\bf Detailed Settings for Numerical Models.} 
We sample expert trajectories with $T = 50$ time steps, 
consider $N = 100$ agents and set the discounted factor $\gamma=0.99$. We set the inverse Boltzmann temperature as $\beta =1$. 
In MFNE expert training, we repeat the fixed point iteration to compute the MF flow. We terminate at the $i$th iteration if the mean squared error over all steps and all state is below or equal to $10^{-10}$, i.e., $$\frac{1}{(T-1)|\Smc|} \sum_{t=1}^{T-1}\sum_{s \in \Smc} \left(\mu^{(i)}_t(s) - \mu^{(i-1)}_t(s)\right)^2 \leq 10^{-10}.$$



\end{document}